\documentclass{article}

\PassOptionsToPackage{numbers, compress}{natbib}

    \usepackage[final]{neurips_2020}

\usepackage[utf8]{inputenc} 
\usepackage[T1]{fontenc}    
\usepackage[hidelinks]{hyperref}       
\usepackage{url}            
\usepackage{booktabs}       
\usepackage{amsfonts}       
\usepackage{nicefrac}       
\usepackage{microtype}      

\usepackage{algorithm}
\usepackage[noend]{algorithmic}
\usepackage{todonotes}
\usepackage{amsfonts}
\usepackage{amsmath}
\usepackage{amsthm}
\usepackage{amssymb}
\usepackage[capitalise]{cleveref}

\usepackage{caption}
\usepackage{subcaption}
\usepackage{wrapfig}
\usepackage{xcolor}
\usepackage{multirow}
\usepackage{enumitem}
\usepackage{float}
\usepackage{xspace}
\usepackage{thm-restate}

\usepackage{printlen}
\uselengthunit{cm}
\newlength\imageheight
\newlength\imagewidth

\newcommand{\backupsmall}{\vspace{-0.1cm}}
\newcommand{\backupmedium}{\vspace{-0.15cm}}

\crefname{equation}{}{} 
\crefname{section}{Sec.}{Sec.}
\crefname{algorithm}{Alg.}{Alg.}
\crefname{ALC@unique}{Line}{Lines}

\DeclareMathOperator*{\argmax}{arg\,max}

\newtheorem{assumption}{Assumption}
\newtheorem{definition}{Definition}

\newcommand{\mypar}[1]{\textbf{#1.}}
\newcommand{\cisr}{\textsc{CISR}\xspace}
\newcommand{\M}{\mathcal{M}}
\newcommand{\Ss}{\mathcal{S}}
\newcommand{\A}{\mathcal{A}}
\newcommand{\D}{\mathcal{D}}
\newcommand{\T}{\mathcal{T}}
\newcommand{\Pp}{\mathcal{P}}
\newcommand{\R}{\mathcal{R}}
\newcommand{\Oo}{\mathcal{O}}
\newcommand{\Z}{\mathcal{Z}}

\title{Safe Reinforcement Learning\\ via Curriculum Induction}

\author{%
  Matteo Turchetta\thanks{The author did part of this work while at Microsoft Research, Redmond.} \\
  Department of Computer Science\\
  ETH Zurich\\
  \texttt{matteotu@inf.ethz.ch} \\
  \And
  Andrey Kolobov \\
  Microsoft Research \\
  Redmond, WA-98052 \\
   \texttt{akolobov@microsoft.com} \\
  \AND
  Shital Shah \\
  Microsoft Research \\
  Redmond, WA-98052 \\
   \texttt{shitals@microsoft.com}
  \And
  Andreas Krause \\
  Department of Computer Science\\
  ETH Zurich \\
  \texttt{krausea@ethz.ch}
  \And
  Alekh Agarwal \\
  Microsoft Research \\
  Redmond, WA-98052 \\
  \texttt{alekha@microsoft.com}
}

\newcommand{\comment}[1]{}

\begin{document}

\setlength{\abovedisplayskip}{2pt}
\setlength{\belowdisplayskip}{2pt}

\maketitle
\begin{abstract}
In safety-critical applications, autonomous agents may need to learn in an environment where mistakes can be very costly. In such settings, the agent needs to behave safely not only \emph{after} but also \emph{while} learning. To achieve this, existing safe reinforcement learning methods make an agent rely on priors that let it avoid dangerous situations during exploration with high probability, but both the probabilistic guarantees and the smoothness assumptions inherent in the priors are not viable in many scenarios of interest such as autonomous driving. This paper presents an alternative approach inspired by human teaching, where an agent learns under the supervision of an automatic instructor that saves the agent from violating constraints during learning. In this new model, the instructor needs to know neither how to do well at the task the agent is learning, nor how the environment works. Instead, it has a library of reset controllers that it activates when the agent starts behaving dangerously, preventing it from doing damage. Crucially, the choices of which reset controller to apply in which situation affect the speed of agent learning. Based on observing agents' progress, the teacher itself learns a policy for choosing the reset controllers, \emph{a curriculum}, to optimize the agent's final policy reward. Our experiments use this framework in two challenging environments to induce curricula for safe and efficient learning. 
\end{abstract}

\backupmedium
\section{Introduction \label{sec:intro}}
\backupsmall
Safety is a major concern that prevents application of reinforcement learning (RL)~\citep{sutton2018reinforcement} to many practical problems~\citep{dulac2019challenges}. Among the RL safety notions studied in the literature \citep{garcia2015comprehensive}, ensuring that the agent does not violate constraints is perhaps the most important. Consider, for exmple, training a policy for a self-driving car's autopilot. Although simulations are helpful, much of the training needs to be done via a process akin to RL on a physical car~\citep{kendall2019driving}. 
At that stage, it is critical to avoid damage to property, people, and the car itself. Safe RL techniques aim to achieve this primarily by imparting the agent with priors about the environment and equipping it with sound ways of updating this information with observations \citep{el2016convex,berkenkamp2017safe,chow2018lyapunov,chow2019lyapunov}. Some do this heuristically \citep{achiam2017constrained}, while others provide safety guarantees through demonstrations \cite{thananjeyan2020safety}, or by assuming access to fairly accurate dynamical models~\citep{akametalu2014reachability}, or at the cost of smoothness assumptions~\citep{berkenkamp2017safe,koller2018learning,turchetta2016safe,turchetta2019safe}. These assumptions hold, e.g., in certain drone control scenarios \citep{berkenkamp2016safe} but are violated in settings such as autonomous driving, where a small delta in control inputs can make a difference between safe passage and collision.
 
In this paper, we propose \textbf{\textit{C}}urriculum \textbf{\textit{I}}nduction for \textbf{\textit{S}}afe \textbf{\textit{R}}einforcement learning (\cisr, \emph{``Caesar''}), a safe RL approach that lifts several prohibitive assumptions of existing ones. \cisr\ is motivated by the fact that, as humans, we successfully overcome challenges similar to those in the autopilot training scenario when we help our children learn safely. Children possess inaccurate notions of danger, have difficulty imitating us at tasks requiring coordination, and often ignore or misunderstand requests to be careful. Instead, e.g., when they learn how to ride a bicycle, we help them do it safely by first equipping the bicycle with training wheels, then following the child while staying prepared to catch them if they fall, finally letting them ride freely but with elbow and knee guards for some time, and only then allowing them to ride like grown-ups. Importantly, each ``graduation" to the next stage happens based on the learner's observed performance under the current safeguard mechanism.

\mypar{Key ideas} In \cisr, an artificial teacher helps an agent (student) learn potentially dangerous skills by inducing a sequence of safety-ensuring training stages called \emph{curriculum}. A student is an RL agent trying to learn a policy for a constrained MDP (CMDP)~\cite{altman1999constrained}. A teacher has a decision rule -- a \emph{curriculum policy} -- for constructing a curriculum for a student given observations of the student's behavior. Each curriculum stage lasts for some number of RL steps of the student and is characterized by an \emph{intervention} (e.g., the use of training wheels) that the teacher commits to use throughout that stage. Whenever the student runs the risk of violating a constraint (falling off the bike), that stage's intervention automatically puts the agent into a safe state (e.g., the way training wheels keep the bike upright), in effect by temporarily overriding the dynamics of the student's CMDP. The teacher's curriculum policy chooses interventions from a pre-specified set such as \{\emph{use of training wheels, catching the child if they fall, wearing elbow and knee guards}\} with the crucial property that any single intervention from this set keeps the agent safe as described above. A curriculum policy that commits to any one of these interventions for the entire learning process is sufficient for safety, but note that in the biking scenario we don't keep the training wheels on the bike forever: at some point they start hampering the child's progress. Thus, the teacher's natural goal is to {\em optimize the curriculum policy} with respect to the student's policy performance at the end of the learning process, assuming the process is long enough that the student's rate of attempted constraint violations becomes very small. In \cisr, the teacher does this via a round-based process, by playing a curriculum policy in every round, observing a student learn under the induced curriculum, evaluating its performance, and trying an improved curriculum policy on a new student in the next round.

\mypar{Related Work} \cisr\ is a form of {\em curriculum learning (CL)} \citep{portelas2020automatic}. CL and {\em learning from demonstration (LfD)} \citep{argall2009ras} are two classes of approaches that rely on a teacher as an aid in training a decision-making agent, but \cisr\ differs from both. In LfD, a teacher provides demonstrations of a good policy for the task at hand, and the student uses them to learn its own policy by behavior cloning \citep{pomerleau89alvinn}, online imitation \citep{Osa_2018}, or apprenticeship learning \citep{abbeel2004irl}. In contrast, \cisr\ does not assume that the teacher has a policy for the student's task at all: e.g., a teacher doesn't need to know how to ride a bike in order to help a child learn to do it. CL generally relies on a teacher to structure the learning process. A range of works~\citep{narvekar2016source, florensa2017reverse, florensa2017automatic, riedmiller2018learning, wu2016training, portelas2019cts,wang2019poet} explore ways of building a curriculum by modifying the learning environment. \cisr\ is closer to \citet{graves2017automated}, which uses a fixed set of environments for the student and also uses a bandit algorithm for the teacher, and to \cite{narvekar2019learning}, which studies how to make the teacher's learning problem tractable. \cisr's major differences from existing CL work is that (1) it is the first approach, to our knowledge, that uses CL for {\em ensuring safety} and (2) uses \emph{multiple} students for training the teacher, which allows it to induce curricula in a more data-driven, as opposed to heuristic, way. Regarding safe RL, in addition to the literature mentioned above,~\citet{le2019batch}, which proposes a CMDP solver, considers the same training and test safety constraints as ours. In that work, the student avoids potentially unsafe environment interaction by learning from batch data, which places strong assumptions on MDP dynamics and data collection policy 
neither verifiable nor easily satisfied in practice~\citep{scherrer2014approximate,chen2019information,agarwal2019optimality}. We use the same solver, but in an online setting.

\looseness -1 The ideas introduced in this work may be applicable in several kinds of safety-sensitive settings where \cisr\ can be viewed as a {\em meta-learning framework} \citep{vanschoren2018meta}, with curriculum policy as a ``hyperparameter" being optimized. In our experiments, the number of iterations \cisr\ needs to learn a good curriculum policy is small. This allows its use in robotics, where a curriculum policy is trained on agents with one set of sensors and applied to training agents with different sensors of similar capabilities, e.g., as in \citet{pan2017agile} for autonomous rovers. Further promising scenarios are training 
a curriculum policy in simulation and applying it to physical agents and using \cisr\ in {\em intelligent tutoring systems} \citep{clement2015}. 

\mypar{Contributions} Our main contributions are: (1) We introduce \cisr, a novel framework for exploiting prior knowledge to guarantee safe training and deployment in RL that forgoes many unrealistic assumptions made in the existing safe RL literature. (2) We present a principled way of optimizing curriculum policies across generations of students while guaranteeing safe student training. (3) We show empirically in two environments that students trained under \cisr-optimized curricula attain reward performance comparable or superior to those trained without a curriculum and remain safe throughout training, while those trained without a curriculum don't. (4) We release an open source implementation of \cisr and of our experiments\footnote{\url{https://github.com/zuzuba/CISR\_NeurIPS20}}.

\backupmedium
\section{Background: Constrained Markov Decision Processes} \label{sec:problem_statement}
\backupsmall
\looseness -1 In this work, we view a learning agent, which we will call a \emph{student},  as performing constrained RL. This framework has been strongly advocated as a promising path to RL safety \citep{ray2019benchmarking}, and expresses safety requirements in terms of an \textit{a priori unknown} set of feasible safe policies that the student should optimize over. In practice, this feasible policy set is often described by a {\em constrained Markov decision process (CMDP)} \citep{altman1999constrained}. We consider CMDPs of the form $\M=\langle \mathcal{S}, \mathcal{A}, \Pp,r, \D\rangle$, where $\mathcal{S}$ and $\mathcal{A}$ are a state and action space, respectively, $\Pp(s'|s,a)$ is a transition kernel, $r:\mathcal{S} \times \mathcal{A}  \times \mathcal{S}\rightarrow \mathbb{R}$ is a reward function, and $\D$ is a set of unsafe terminal states. We focus on settings where safety corresponds to avoiding visits to the set $\D$. The student's objective, then, is to find a policy $\pi:\mathcal{S}\rightarrow \Delta_{\mathcal{A}}$, i.e., a mapping from states to action distributions, that solves the following constrained optimization problem, where $\rho^\pi$ is a distribution of trajectories induced by $\pi$ and $\Pp$ given some fixed initial state distribution: 
\begin{equation}
    \pi^*=\argmax_\pi~ \mathbb{E}_{\rho^{\pi}} \sum_{t=0}^T r(s_t, a_t, s_{t+1}),~~~ \textrm{s.t.}~~~ \mathbb{E}_{\rho^{\pi}} \sum_{t=0}^T \mathbb{I}(s_t\in \D) \leq \kappa, \label{eq:CMDP}
\end{equation}

where $\mathbb{I}$ is the indicator function. To ensure complete safety, we restrict our attention to problems where the value of $\kappa$ makes the constraints feasible. 
While we have presented the finite-horizon undiscounted version of the problem, both the objective and the constraint can be expressed as the average or a discounted sum over an infinite horizon. For a generic CMDP $\M$ we denote the set of its feasible policies as $\Pi_{\M}$, and the value of any $\pi \in \Pi_{\M}$ as $V_{\M}(\pi)$.

There are a number of works on solving CMDPs that find a nearly feasible and optimal policy with sufficiently many trajectories, but violate constraints during training~\cite{achiam2017constrained,chow2019lyapunov}. In contrast, we aim to enable the student to learn a policy for CMDP $\M$ {\em without violating any constraints in the process}.

\backupmedium
\section{Curriculum Induction for Safe RL \label{sec:cisr}}
\backupsmall

\looseness -1 We now describe \cisr, our framework for enabling a student to learn without violating safety constraints. To address the seemingly impossible problem of students learning safely in an unknown environment, \cisr\ includes a \emph{teacher}, which is a learning agent itself. The teacher serves two purposes: (1) mediating between the student and its CMDP in order to keep the student safe, and (2) learning a mediation strategy -- a curriculum policy -- that helps the student learn faster. First we formally describe the teacher's mediation tools called \emph{interventions} (Section \ref{sec:inter}). Then, in Section \ref{sec:student} we show how, from the student's perspective, each intervention corresponds to a special CMDP where training is safe and every feasible policy is also feasible in original CMDP, ensuring that objective (1) is always met. Finally, in Section \ref{sec:Tproblem} we consider the teacher's perspective and show how, in order to optimize objective (2), it iteratively improves its curriculum policy by trying it out on different students.

\subsection{Interventions \label{sec:inter}}

In \cisr, the teacher has a set of \emph{interventions} $\mathcal{I} = \{\langle \D_i, \T_i\rangle\}_{i=1}^K$. Hereby, each intervention is defined by a set $\D_i \subset \mathcal{S}$ of \emph{trigger states} where this intervention applies and $\T_i: \mathcal{S} \rightarrow \Delta_{\mathcal{S} \setminus \D_i}$, a state-conditional reset distribution. The semantics of an intervention is as follows. From the teacher's perspective, $\D_i$ is a set of undesirable states, either because $\D_i$ intersects with the student CMDP's unsafe state set $\D$ or because the student's current policy may easily lead from $\D_i$'s states to $\D$'s. Whenever the student enters a state $s \in \D_i$, the teacher can \emph{intervene} by resetting the student to another, safe state according to distribution $\T_i(\cdot|s)$. We assume the following about the interventions:

\begin{assumption}[\textbf{\emph{Intervention set}}] a) The intervention set $\mathcal{I}$ is given to the teacher as input and is fixed throughout learning. b) The interventions in $\mathcal{I}$ cannot be applied after student learning.
\end{assumption}

Assumption a) is realistic in many settings, where the student is kept safe by heuristics in the form of simple controllers such as those that prevent drones from stalling. At least one prior work, \citet{eysenbach2017leave}, focuses on how safety controllers can be learned, although their safety notion (policy reversibility) is much more specialized than in \cisr. Assumption b) is realistic in that a safety controller practical enough to be used beyond training is likely to be part of the agent as in \cite{alshiekh2018shielding}, removing the need for safety precautions during training. An example of a safety mechanism that cannot be used beyond training is motion capture lab equipment to prevent collisions. 


\subsection{The student's problem \label{sec:student}}

We now describe the student's learning process under single and multiple interventions of the teacher. As we explain here, training in the presence of an interventions-based teacher can be viewed as learning in a {\em sequence} of CMDPs that guarantee student safety under simple conditions. Later, in \cref{sec:Tproblem}, we formalize these CMDP sequences as {\em curricula}, and show how the teacher can induce them using a {\em curriculum policy} in order to accelerate students' learning progress.

\textbf{Intervention-induced CMDPs.} Fix an intervention $\langle \D_i, \T_i \rangle$ and suppose the teacher commits to using it throughout student learning. As long as the student avoids states in $\D_i$, deemed by the teacher too dangerous for the student's ability, the student's environment works like the original CMDP, $\mathcal{M}$. But whenever the student enters an $s \in \D_i$, the teacher leads it to a safe state $s' \sim \T_i(\cdot|s), s' \notin \D_i$. 

\looseness -1 Thus, each of teacher's interventions $i\in \mathcal{I}$ induces a student CMDP $\M_i=\langle \mathcal{S}, \A, \Pp_i,r_i,\D, \D_i \rangle$, where $\mathcal{S}$ and $\A$ are as in the original CMDP $\M$, but the dynamics are different: for all $a \in \A$, $\Pp_i(s'|s,a)=\Pp(s'|s,a)$ for all $s \in \mathcal{S} \setminus \D_i$  and $\Pp_i(s'|s,a) = \T_i(s'|s)$ for $s \in \D_i$. The reward function is modified to assign $r_i(s,a, s') = 0$ for $s \in \D_i, s' \notin \D_i$: all student's actions in these cases get overridden by the teacher's intervention, having no direct cost for the student. However, the teacher cannot supervise the student forever; the student must learn a safe and high-return policy that does not rely on its help. We thus introduce a constraint on the number of times the student can use the teacher's help. This yields the following problem formulation for the student, where $\rho_i^\pi$ is a distribution of trajectories induced by $\pi$ given some fixed initial state distribution and the modified transition function $\Pp_i$:
\begin{align}
\hspace*{-0.22cm}
\pi^*=\argmax \mathbb{E}_{\rho_i^\pi} \sum_{t=0}^T r_i(s_t, a_t, s'_{t+1}),~ \textrm{s.t.}~~ \mathbb{E}_{\rho_i^\pi} \sum_{t=0}^T \mathbb{I}(s_t \in \D) \leq \kappa_i, ~ \mathbb{E}_{\rho_i^\pi} \sum_{t=0}^T \mathbb{I}(s_t \in \D_i) \leq \tau_i
\label{eq:CMDP_intervention},
\end{align}
where $\kappa_i \geq 0$ and $\tau_i \geq 0$ are intervention-specific tolerances set by the teacher. Thus, although the student doesn't incur any cost for the teacher's interventions, they are associated with violations of teacher-imposed constraints. Our CMDP solver \cite{le2019batch}, discussed in Section \ref{sec:impl}, penalizes the student for them during learning, making sure that the student doesn't exploit them in its final policy.

By construction, each intervention-induced CMDP $\M_i$ has two important properties, which we state below and prove in \cref{sec:app:proof}. First, if the teacher has lower tolerance for constraint violations than the original CMDP, an optimal learner operating in $\M_i$ will eventually come up with a policy that is safe in its original environment $\M$:

\begin{restatable}[\textbf{\emph{Eventual safety}}]{proposition}{SafeDeployment} \label{prop:feasibility}
    Let $\Pi_\mathcal{M}$ and $\Pi_{\mathcal{M}_i}$ be the sets of feasible policies for the problems in Equations \cref{eq:CMDP} and \cref{eq:CMDP_intervention}, respectively. Then, if $\tau_i + \kappa_i \leq \kappa$, $\Pi_{\mathcal{M}_i} \subseteq \Pi_{\mathcal{M}}$.
\end{restatable}

Intuitively, once the teacher is removed, the student can fail either by passing through states where it used to be rescued or by reaching states where the teacher was not able to save it in the first place; $\tau_i + \kappa_i \leq \kappa$ ensures that the probability of either of these cases is sufficiently low that a feasible policy in $\M_i$ is also feasible in $\M$.
While this  guarantees the student's \emph{eventual} safety, it doesn't say anything about {\em safety during learning}. The second proposition states conditions for learning safely:

\begin{restatable}[\textbf{\emph{Learning safety}}]{proposition}{SafeTraining} \label{prop:learning_safety}
   Let $\D$ be the set of unsafe states of CMDPs $\M$ and $\M_i$, and let $\D_i$ be the set of trigger states of intervention $i$. If $\D \subseteq \D_i$ and $\Pp(s'|a, s) = 0$ for every $s' \in \D$, $s \in \mathcal{S} \setminus \D_i$, and $a \in \A$, then an optimal student learning in CMDP $\M_i$ will not violate any of $\M$'s constraints throughout learning.
\end{restatable}

Informally, \cref{prop:learning_safety} says that if the set of trigger states $\D_i$ of the teacher's intervention ``blankets" the set of unsafe states $\D$, the student has no way of reaching states in $\D$ without triggering the intervention and being rescued first, and hence is safe even when it violates $\M_i$'s constraints.

\begin{assumption}[\textbf{\emph{Intervention safety}}]
In the rest of the paper, we assume all teacher interventions to meet the conditions of Proposition \ref{prop:learning_safety}.
\end{assumption}

We make this assumption for conceptual simplicity, but it can hold in reality: systems such as aircraft stall
prevention and collision avoidance guarantee near-absolute safety. Even in the absence thereof, \cisr informally keeps
the student as safe during training as teacher’s interventions allow. In \cref{sec:experiments}, we show that, even when this assumption is violated and the interventions cannot guarantee absolute safety, \cisr still improves training safety by three orders of magnitude over existing approaches.

\mypar{Sequences of intervention-induced CMDPs and knowledge transfer} As suggested by the biking example, the student's learning is likely to be faster under a \emph{sequence} of teacher interventions, resulting in a sequence of CMDPs $\M_{i_1}, \M_{i_2}, \ldots$. This requires a mechanism for the student to carry over previously acquired skills from one CMDP to the next. We believe that most knowledge transfer approaches for unconstrained MDPs, such as transferring samples \citep{lazaric2011transfer}, policies \citep{fernandez2010probabilistic}, models \citep{fachantidis2013transferring} and values \citep{taylor2005behavior}, can be applied to CMDPs as well, with the caveat that the transfer mechanism should be tailored to the environment, teacher's intervention set, and the learning algorithm the student uses. In Section~\ref{sec:impl}, we present the knowledge transfer mechanism used in our implementation.\\

\begin{wrapfigure}{R}{0.55\textwidth}
\vspace{-0.30in}
\begin{minipage}[t]{0.55\textwidth}
\centering
\begin{algorithm}[H]
\caption{\cisr}
\label{alg:cisr}
\begin{algorithmic}[1] 
        \STATE \textbf{Input}: Interventions $\mathcal{I}$, Initial teacher $\pi^T_0$
        \FOR {$j=0,1,\ldots, N_t$}
            \STATE $\pi_{0,j} \gets \texttt{get\_student}()$ \label{alg:line:Treset}
            \FOR{$n=0,1,\ldots,N_s$} \label{alg:line:InnerLoop}
                \STATE $\M_{i_n} \gets \pi^T_j(o^T_0,\ldots,o^T_n)$  \label{alg:line:Taction}
                \STATE \textbf{if} $n > 0$ \textbf{then} $\pi_{n,j} \gets \texttt{transfer} (\pi_{n-1,j})$ \label{alg:line:Stransfer}
                \STATE $\pi_{n,j} \gets \texttt{student.train}(\M_{i_n})$  
                \label{alg:line:Strain}
                \STATE $o^T_{n} \gets \phi(\pi_{n,j})$ \label{alg:line:Tobs}\label{alg:line:return}
                %
            \ENDFOR 
        \STATE $\pi^T_{j+1} \gets \texttt{teacher.train}(\{(\pi^T_k,\hat{V}(\pi_{N_s,k}))\}_{k=1}^j)$ \label{alg:Tlearn:line:Ttrain}
        \ENDFOR
    \end{algorithmic}
%
\end{algorithm}
\end{minipage}
\vspace{-0.2in}
\end{wrapfigure}

\subsection{The teacher's problem} \label{sec:Tproblem}

Given that, under simple conditions, any sequence of teacher's interventions will keep the student safe, the teacher's task is to sequence interventions/CMDPs for the student so that the student learns the highest-expected-reward policy. In \cisr, the teacher does this by iteratively trying a curriculum policy on different students and improving it after each attempt. This resembles human societies, where curriculum policies are implicitly learned through educating generations of students. At the same time, it is different from prior approaches such as \citet{graves2017automated,matiisen2017teacherstudent}, which try to learn and apply a curriculum on the same student. These approaches embed a fixed, heuristic curriculum policy within the teacher to induce a curriculum but cannot improve this policy over time. In contrast, \cisr exploits information from previous students to optimize its curriculum policy following a data-driven approach.

\mypar{What does \emph{not} need to be assumed of the teacher} \cisr\ makes very few assumptions about the teacher's abilities. 
In particular, while the teacher needs to be able to evaluate the student's progress, it can do so according to its internal notion of performance, which may differ from the student's one. Thus, the teacher does not need to know the student's reward. Moreover, since the interventions are given, the teacher does not need to know the student's dynamics. However, the intervention design may still require \emph{approximate} and \emph{local} knowledge of the dynamics, which is much less restrictive than \emph{perfect} and \emph{global} knowledge. Furthermore, the teacher does not need to have a policy for performing the task that the student is trying to learn nor to be able to communicate the set $\D_i$ of an intervention's trigger states for any $i$ to the student. It only needs to be able to {\em execute} any intervention in $\mathcal{I}$ without violating laws governing the CMPD $\M$, i.e., using conditional reset distributions that are realizable according to $\M$'s dynamics. However, the teacher is not assumed to use only the student's action and observation set $\A$ and $\Ss$ to execute the interventions --- it may be able to do things the student  can't, such as setting the the student upright if it is falling from a bike.  

\mypar{In \cisr\, the teacher learns online} Abstractly, we view the teacher as an online learner presented in Algorithm \ref{alg:cisr}.
In particular, for rounds $j = 1, \ldots, N_t$:
\begin{enumerate}[leftmargin=*]
    \item The teacher plays a \emph{decision rule} $\pi^T_j$ that makes a \emph{new} student $j$ learn under an adaptively constructed sequence  
     $C_j=(\M_{i_1},\ldots,\M_{i_{N_s}})$ of intervention-induced CMDPs (lines \ref{alg:line:InnerLoop}-\ref{alg:line:Tobs}). 
    \item Each student $j$ learns via a total of $N_s$ interaction units (e.g., steps, episodes, etc.) with an environment. During each unit, it acts in a CMDP in $C_j$. It updates its policy by transferring knowledge across interaction units (\cref{alg:line:Stransfer}). The teacher computes \emph{features} $\phi(\pi_{n,j})$ of student $j$'s performance (lines \ref{alg:line:Tobs}) by evaluating $j$'s policies throughout $j$'s learning process. Based on them, the teacher's decision rule proposes the next intervention MDP in $C_j$. 
    \item The teacher adjusts its decision rule's parameters (line \ref{alg:Tlearn:line:Ttrain}) that govern how a CMDP sequence $C_{j+1}$ will be produced in the next round.
\end{enumerate}

\begin{assumption}[\textbf{\emph{Length of student learning}}] For all potential students, their CMDP solvers $\texttt{student.train}$ are sound and complete.\footnote{That is, it finds a feasible policy after enough interactions with the CMDP, assuming one exists.} $N_s$ is much larger than the number of interactions it takes the solver to find a feasible policy $\pi \in \Pi_{\M_i}$ for at least one intervention CMDP $\M_i$, $i \in \mathcal{I}$.\end{assumption}

This assumption ensures that students can, in principle, learn a safe policy in the allocated amount of training $N_s$ under some intervention sequence. It also allows the teacher to learn to induce such sequences, given enough rounds $N_t$, even though not every student trained in the process will necessarily have a feasible policy for $\M$ at the end of its learning.

This framework's concrete instantiations depend on the specifics of (i) the decision rule that produces a sequence $C_j$ in each round, (ii) the teacher's evaluation of the student to estimate $\hat{V}(\pi_{N_s,k})$ in each round. 
Next, we consider each of these aspects. \\

\mypar{Curricula and curriculum policies (i)} Before discussing teacher's decision rules that induce intervention sequences in each round, we formalize the notion of these sequences themselves:

\begin{definition}[\textbf{\emph{Curriculum}}] \label{def:curr}
Suppose a student learns via $N_s$ interaction units (e.g., steps or episodes) with an environment, and 
let $\mathcal{I}$ be a set of teacher's interventions. A curriculum $C$ is a sequence $\M_{i_1}, \ldots, \M_{i_{N_s}}$ of length $N_s$ of CMDPs s.t. the student interacts with CMDP $\M_{i_n}$ during unit $n$, where $\M_{i_n}$ is induced by an intervention $i_n \in \mathcal{I}$.
\end{definition}

The difference between a curriculum and a teacher's decision rule that produces it is crucial. While a decision rule for round $j$ \emph{can} be a mapping $C_j:[N_s] \rightarrow \mathcal{I}$ exactly like a curriculum, in general it is useful to make it depend on the student's policy $\pi_{n,j}$ at the start of each interaction unit $n$. In practice, the teacher doesn't have access to $\pi_{n,j}$ directly, but can gather some statistics $\phi(\pi_{n,j})$ about it by conducting an evaluation procedure that we discuss shortly. Examples of useful statistics include the number of times the student's policy triggers teacher's interventions, features of states where this happens, and, importantly, an estimate of the policy value  $\hat{V}(\pi_{n,j})$.

Thus, an adaptive teacher is an agent operating in a \emph{partially observable} MDP $\langle \Ss^T, \A^T, \Pp^T, \R^T, \Oo^T, \Z^T \rangle$, where $\Ss^T = \overline{\Pi}_{\M}$  is the space of \emph{all} student policies for the original CMDP $\M$ (not only feasible ones), $\A^T = \mathcal{I}$ is the set of all teacher interventions, $\Pp^T: \overline{\Pi}_{\M} \times \mathcal{I} \times \overline{\Pi}_{\M} \rightarrow [0,1]$ is governed by the student's learning algorithm, $\Oo^T = \Phi$ is the space of evaluation statistics the teacher gathers, and $\Z^T = \phi$ 
is the mapping from the student's policies to statistics about them, governed by the teacher's evaluation setup. The reward function $\R^T$ can be defined as $\R^T(n) = \hat{V}(\pi_{n,j}) - \hat{V}(\pi_{n-1,j})$, with $\R^T(0) =\hat{V}(\pi_{0,j})$ the ``progress" in student's policy quality from one curriculum stage to the next. Note, however, that what really matters to the teacher is the student's perceived policy quality at the end of round, 
$\hat{V}(\pi_{N_s, j}) = \sum_{n=1}^{N_s} \R^T(n)$. Thus, in general, a teacher's decision rule is a solution to this POMDP:

\begin{definition}[\textbf{\emph{Curriculum policy}}] Let $\mathcal{H}$ be the space of teacher's observation histories. A curriculum policy is a mapping $\pi^T:\mathcal{H} \rightarrow \mathcal{I}$ that, for any $n\in [N_s]$, specifies an intervention given the teacher's observation history $\phi(\pi_1), \ldots, \phi(\pi_{n-1})$ at the start of the student's $n$-th interaction unit.
\end{definition}

In the context of curriculum learning, modeling the teacher as a POMDP agent similar to ours was proposed in \cite{matiisen2017teacherstudent,narvekar2019learning} -- though not for RL safety. However, from the computational standpoint, \cisr's view of a teacher as an online learning agent captures a wider range of possibilities for the teacher's practical implementation. 
For instance, it suggests that it is equally natural to view the teacher as a bandit algorithm that plays suitably parameterized curriculum policies in each round, which is computationally much more tractable than using a full-fledged POMDP solver. As described in Section \ref{sec:impl}, this is the approach we take in this work. \\

\mypar{Safely evaluating students' policies (ii)} Optimizing the curriculum policy requires the evaluation of students' policies to create features and rewards for the teacher. Since a student's intermediate policy may not be safe w.r.t. $\M$, evaluating it in $\M$ could lead to safety violations. 
Instead, we assume that the teacher's intervention set $\mathcal{I}$ includes a special intervention $i_0$ satisfying Proposition \ref{prop:learning_safety}. Hence, this intervention induces a CMDP $\M_0$ where the student can't violate the safety constraint. The teacher uses it for evaluation, although it can also be used for student learning. Therefore, for any policy $\pi$ over the state-action space $\Ss\times\A$, we define $\hat{V}(\pi) \triangleq V_{\M_0}(\pi)$. 
Since $\M_0$ has more constraints than $\M$, evaluation in $\M_0$ underestimates the quality of the teacher's curriculum policy. Let $\hat{\pi}^*$ be the optimal policy for $\M_0$ and $\pi^*$ the optimal policy for $\M$. The value of the teacher's curriculum policy is, then, $\hat{V}(\hat{\pi}^*) \leq V_{\M}(\pi^*)$.  If the student policy violates a constraint in $\M_0$ during execution, the teacher gets a reward $-2TR_{\max}$ where $R_{\max}$ is the largest environment reward.
\backupmedium
\section{Implementation Details \label{sec:impl}}
\backupsmall
\cisr\ allows for many implementation choices. Here, we describe those used in our experiments.

\mypar{Student's training and knowledge transfer} Our students are CMDP solvers based on \cite{le2019batch}, but train online rather than offline as in \cite{le2019batch} since safety is guaranteed  by the teacher. This is a primal-dual solver, where the primal consists of an unconstrained RL problem including the original rewards and a Lagrange multiplier penalty for constraint violation. The dual updates the multipliers to increase the penalty for violated constraints. We use the Stable Baselines \citep{stable-baselines} implementation of PPO \citep{schulman2017proximal} to optimize the Lagrangian of a CMDP for a fixed value of the multipliers, and Exponentiated Gradient \citep{kivinen1997exponentiated}, a no-regret online optimization algorithm, to adapt the multipliers. Our students transfer both value functions and policies across interventions, but reset the state of the optimizer.

\mypar{Teacher's observation} Before every switch to a new intervention $i_{n+1}$, our teacher evaluates the student's policy in CMDP $\M_{i_{n}}$ induced by the previous intervention. The features estimated in this evaluation, which constitute the teacher's observation $o^T_{n}$, are $V_{\M_{i_{n}}}(\pi)$, the student's policy value in CMDP $\M_{i_{n}}$, and the rate of its constraint violation there, $\mathbb{E}_{\rho_{i_n}^\pi} [\sum_{t=0}^T \mathbb{I}(s_t \in \D_{i_n}) - \tau_{i_n}]$.

\mypar{Reward} 
As mentioned in \cref{sec:Tproblem}, from a computational point of view it is convenient to approach the curriculum policy optimization problem as an online optimization problem for a given parametrization of the teacher's curriculum policy. This is the view we adopt in our implementation, where in round $j$, the teacher's objective is the value of the student's final policy $\hat{V}(\pi_{N_s, j})$. Moreover, since after $N_s$ curriculum steps the student's training is over, we compute the student's return directly in $\M$ rather than using a separate evaluation intervention $\M_0$.

\mypar{Policy class}
To learn a good teaching policy efficiently, we restrict the teacher's search space to a computationally tractable class of parameterized policies. We consider reactive policies that depend only on the     teacher's current observation, $o^T_n$, so $\pi^T(o^T_0, o^T_1,\ldots,o^T_n)=\pi^T(o^T_n)$. Moreover, we restrict the number of times the teacher can switch their interventions to at most $K\leq N_s$. A policy from this class is determined by a sequence of $K$ interventions and by a set of rules that determines when to switch to the next intervention in the sequence. Here, we consider simple rules that require the average return and constraint violation during training to be greater/smaller than a threshold. Formally, we denote the threshold array as $\omega \in \mathbb{R}^2$. The teaching policy we consider switches from the current intervention to the next when $\phi(\pi_{n,j})[0] \geq \omega[0] \land \phi(\pi_{n,j})[1] \leq \omega[1]$. Thus, teacher's policies are fully determined by $3K+1$ parameters. 

\mypar{Teacher's training with GP-UCB} Given the teacher's policy is low dimensional and sample efficiency is crucial, we use Bayesian optimization (BO) \citep{mockus1978application} to optimize it. That is, we view the $3K+1$ dimensional parameter vector of the teacher as input to the BO algorithm, and search for parameters that approximately optimize the teacher's reward. Concretely, we use GP-UCB \citep{srinivas2009gaussian}, a simple Bayesian optimization algorithm that enjoys strong theoretical properties.
\backupmedium
\section{Experiments} \label{sec:experiments}
\backupsmall
We present experiments where \cisr\ efficiently and safely trains deep RL agents in two environments: the \textit{Frozen Lake} and the \textit{Lunar Lander} environments from Open AI Gym \citep{brockman2016openai}. While \textit{Frozen Lake} has simple dynamics, it demonstrates how safety exacerbates the difficult problem of exploration in goal-oriented environments. \textit{Lunar Lander} has more complex dynamics and a continuous state space. We compare students trained with a curriculum optimized by \cisr\  to students trained with trivial or no curricula in terms of safety and sample efficiency. \emph{In addition, we show that curriculum policies can transfer well to students of different architectures and sensing capabilities (\cref{tab:lander}).} \textbf{For a detailed overview of the hyperparameters and the environments, see \cref{sec:app:hyperpar,sec:app:environments}.}

\mypar{\textit{Frozen Lake}} In this grid-world environment (\cref{fig:flake_small_map}), the student must reach a goal in a 2D map while avoiding dangers. It can move in 4 directions. With probability $80\%$ it moves in the desired direction and with $10\%$ probability it moves in either of the orthogonal ones. The student only sees the map in \cref{fig:flake_small_map} and is not aware of the teacher interventions' trigger states (\cref{fig:flake_small_map_with_teacher1,fig:flake_small_map_with_teacher2}). Note that the high density of obstacles, the strong contrast between performance and safety, the safe training requirement and the high-dimensional observations (full map, as opposed to student location) make this environment substantially harder than the standard \textit{Frozen Lake}.

We use three interventions, whose trigger states are shown in \cref{fig:flake_small_map_with_teacher1,fig:flake_small_map_with_teacher2}: \textit{soft reset 1 (SR1)}, \textit{soft reset 2 (SR2)}, and \textit{hard reset (HR)}. \textit{SR1} and \textit{SR2} have tolerance $\tau=0.1$ and reset the student to the state where it was the time step before being rescued by the teacher. \textit{HR} has zero tolerance, $\tau=0$ and resets the student to the initial state.

We compare six different teaching policies: (\textit{i}) \textit{No-intervention}, where students learn in the original environment; (\textit{ii}-\textit{iii}-\textit{iv}) \textit{single-intervention}, where students learn under each of the interventions fixed for the entire learning duration; (\textit{v}) \textit{Bandit}, where students follow curricula induced by the \textit{a priori} fixed curriculum policy from \cite{matiisen2017teacherstudent};
(\textit{vi}) \textit{Optimized}, where we use the curriculum policy found by \cisr\ within the considered policy class after interacting with 30 students. We let each of these curriculum policies train 10 students. For the purposes of analysis, we periodically freeze the students' policies and evaluate them in the original environment. 

\mypar{Results} \cref{fig:flake_curves} shows the success rate and the return of the students' policies deployed in the original environment as training progresses. Without the teacher's supervision (\textit{No-interv.}), the students learn sensible policies. However, the training is slow and results in thousands of failures (\cref{tab:flake} in \cref{sec:app:environments}). The \textit{HR} intervention resets the students to the initial state distribution whenever the teacher constraint is violated. However, since it has more trigger states than \textit{No-interv.} (\cref{fig:flake_small_map_with_teacher1} vs \cref{fig:flake_small_map}), the students training exclusively under it do not explore enough to reach the goal. \textit{SR1} and \textit{SR2} allow the student to learn about the goal without incurring failures thanks to their reset distribution, which is more forgiving that \textit{HR}'s one. However,  they result in performance plateaus as the soft reset of the training environment lets the students recover from mistakes in a way the deployment environment doesn't. The \textit{Optimized} curriculum retains the best of both worlds by initially proposing a soft reset intervention that allows the agent to reach the goal and subsequently switching to the hard reset such that the training environment is more similar to the original one. Finally, the \textit{Bandit} curriculum policy from \cite{matiisen2017teacherstudent} requires an initial exploration of different interventions with each student as it does not learn across students. This is in contrast with \cisr, which exploits information acquired from previous students to improve its curriculum policy, and results in slower training. \cref{tab:flake} in \cref{sec:app:environments} shows the confidence intervals of mean performance across 10 students and teachers trained on 3 seeds, indicating \cisr's robustness.


\begin{figure}
\begin{minipage}[t]{0.56\textwidth}
     \centering
     \begin{minipage}[t]{0.26\linewidth}
         \centering
         \includegraphics[width=\textwidth]{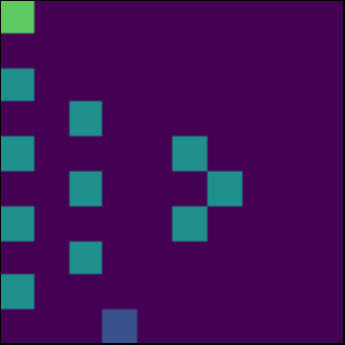}
         \subcaption{Original map}
         \label{fig:flake_small_map}
     \end{minipage}
     \begin{minipage}[t]{0.26\linewidth}
         \centering
         \includegraphics[width=\textwidth]{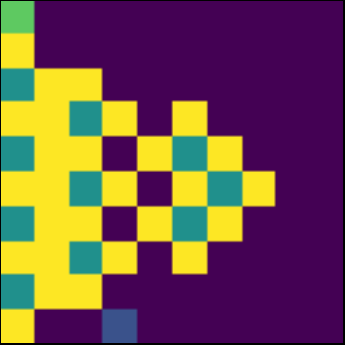}
         \subcaption{\small \textit{SR1} and \textit{HR}}
         \label{fig:flake_small_map_with_teacher1}
     \end{minipage}
     \hspace{-0.9cm}
     \begin{minipage}[t]{0.52\linewidth}
         \centering
         \includegraphics[width=1.02\textwidth]{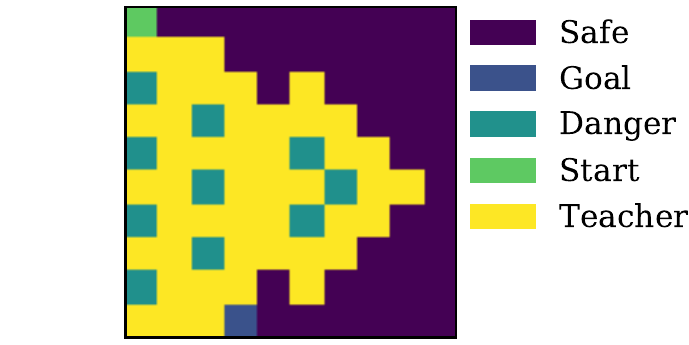}
         \subcaption{\small \textit{SR2}}
         \label{fig:flake_small_map_with_teacher2}
     \end{minipage}
     \vspace{-0.2cm}
     \caption{\small Interventions for \textit{Frozen Lake}. Maps \ref{fig:flake_small_map_with_teacher1} and \ref{fig:flake_small_map_with_teacher2} show trigger state sets $\mathcal{D}_{\mathit{SR1}}$ and $\mathcal{D}_{\mathit{SR2}}$ for interventions $\mathit{SR1}$ and $\mathit{SR2}$, which get triggered at distance = $1$ and $2$ from lakes (dangers), respectively. Intervention $\mathit{HR}$ has $\mathcal{D}_{\mathit{HR}} = \mathcal{D}_{\mathit{SR1}}$.}
\end{minipage}
\hspace{0.3cm}
\begin{minipage}[t]{0.36\textwidth}
     \centering
     \begin{minipage}[t]{0.47\textwidth}
         \centering
         \includegraphics[width=\textwidth]{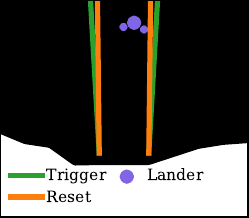}
         \subcaption{``Narrow"}
     \end{minipage}
     \begin{minipage}[t]{0.47\textwidth}
         \centering
         \includegraphics[width=\textwidth]{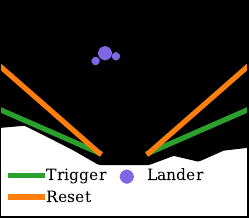}
         \subcaption{``Wide"}
    \end{minipage}
    \vspace{-0.2cm}
    \caption{\small Interventions for \textit{Lunar Lander}. If the student hits the green line, it gets reset to a state on the orange line. See \cref{sec:app:environments} for more details.}
    \label{fig:funnel}
\end{minipage}
\vspace{-0.5cm}
\end{figure}

\begin{figure}[t]
     \centering
     \begin{subfigure}[t]{0.415\textwidth}
         \centering
         \includegraphics[width=\textwidth]{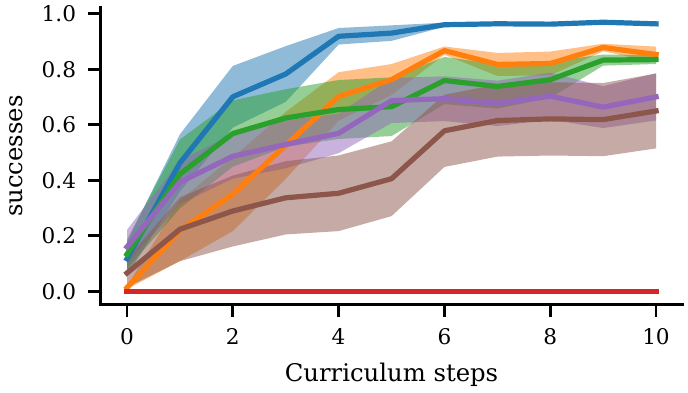}
     \end{subfigure}
     \hfill
     \begin{subfigure}[t]{0.415\textwidth}
         \centering
         \includegraphics[width=\textwidth]{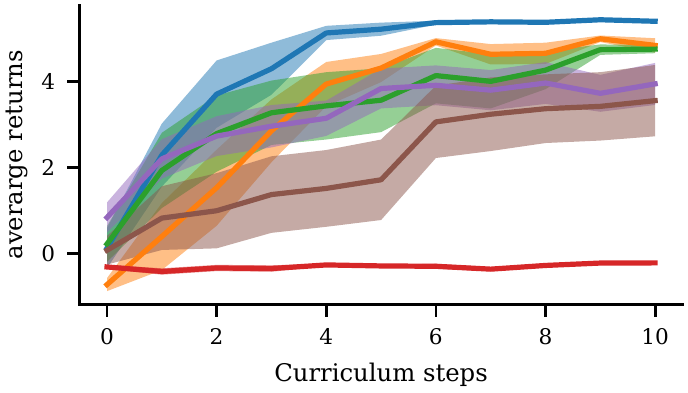}
    \end{subfigure}
    \hfill
    \begin{subfigure}[t]{0.15\textwidth}
         \centering
         \includegraphics[width=\textwidth]{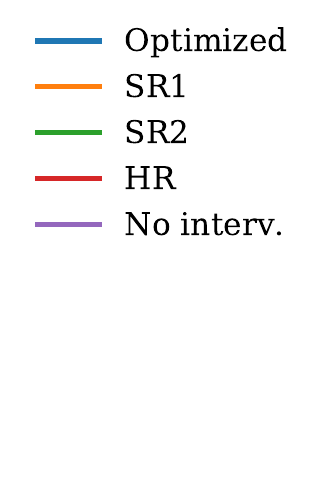}
    \end{subfigure}
    \caption{\small Student success rate (\textbf{Left}) and average returns (\textbf{Right}) in \textit{Frozen Lake} as student training under different curriculum policies progresses. The \textit{Optimized} curriculum policy outperforms training in the original environment (\textit{No-interv.}), in each of the individual interventions (\textit{SR1}, \textit{SR2}, \textit{HR}) and the pre-specified curriculum policy of \cite{matiisen2017teacherstudent} (\textit{Bandit}).}
    \label{fig:flake_curves}.
    \vspace{-0.6cm}
\end{figure}

\mypar{\textit{Lunar Lander}} In this environment, the goal is to safely land a spaceship on the Moon. Crucially, the Moon surface is rugged and differs across episodes. The only constant is a flat landing pad that stretches across $\mathcal{X}_{land}$ along the $x$ dimension (at $y=0$). Landing safely is particularly challenging since agents do not observe their distance from the ground, only their absolute $x$ and $y$ coordinates.

\looseness -1 Since the only way to ensure safe training is to land on the landing pad each time, we use the two interventions in \cref{fig:funnel}. Each gets triggered based on the student's tilt angle and $y$-velocity if the student is over the landing pad ($x \in \mathcal{X}_{land}$), and on a funnel-shaped function of its $x,y$ coordinates otherwise ($x \not \in \mathcal{X}_{land}$). The former case prevents crashing onto the landing pad, the latter, landing (and possibly crashing) on the rugged surface the student cannot sense. We call the interventions \textit{Narrow} and \textit{Wide} (see \cref{fig:funnel}); both set the student's velocity to 0 after rescuing it to a safe state. The interventions, despite ensuring safety, make exploration harder as they make experiencing a natural ending of an episode difficult.

We compare four curriculum policies: (\textit{i}) \textit{No-intervention}, (\textit{ii}-\textit{iii}) \textit{single-intervention} and (\textit{iv}) \textit{Optimized}. We let each policy train 10 students and we compare their final performance in the original \textit{Lunar Lander}. Moreover, we use the \textit{Optimized} curriculum to train different students than those it was optimized for, thus showing the transferability of curricula.

\mypar{Results} \cref{tab:lander} (left) shows for each curriculum policy the mean of the students' final return, success rate and failure rate in the original environment and the average number of failures during training. The \textit{Narrow} intervention makes exploration less challenging but prevents the students from experiencing big portions of the state space. Thus, it results in fast training that plateaus early. The \textit{Wide} intervention makes exploration harder but it is more similar to the original environment. Thus, it results in slow learning. \textit{Optimized} retains the best of both worlds by initially using the \textit{Narrow} intervention to speed up learning and subsequently switching to the \textit{Wide} one. In \textit{No-interv.}, exploration is easier since the teacher's safety considerations do not preclude a natural ending of the episode. Therefore, \textit{No-interv.} attains a comparable performance to \textit{Optimized}. However, the absence of the teacher results in three orders of magnitude more training failures. 

\cref{tab:lander} shows the results of using the teaching policy optimized for students with 2-layer MLP policies and noiseless observations to train students with noisy sensors (center) and different architectures (right). These results are similar to those previously observed: the \textit{Optimized} curriculum attains a comparable performance to the \textit{No interv.} training while greatly improving safety. This shows that teaching policies can be effectively transferred, which is of great interest in many applications. \cref{tab:landerCI} in \cref{sec:app:environments} shows the confidence intervals for these experiments over 3 seeds.


\begingroup
\setlength{\tabcolsep}{1pt} 
\begin{table}[!]
    \footnotesize
    \centering
    \begin{tabular}{l|c c c c | c c c c | c c c c}
         & \multicolumn{4}{c}{\textbf{Eval. on noiseless, 2-layer student}} & \multicolumn{4}{c}{\textbf{Eval. on noisy student}} & \multicolumn{4}{c}{\textbf{Eval. on 1-layer student}} \\
         & $V_{\M}(\pi_{N_s})$ & Succ. & \vtop{\hbox{\strut Test}\hbox{\strut fail}} & \vtop{\hbox{\strut Train}\hbox{\strut fail}} & $V_{\M}(\pi_{N_s})$ & Succ. & \vtop{\hbox{\strut Test}\hbox{\strut fail}} & \vtop{\hbox{\strut Train}\hbox{\strut fail}} & $V_{\M}(\pi_{N_s})$ & Succ. & \vtop{\hbox{\strut Test}\hbox{\strut fail}} & \vtop{\hbox{\strut Train}\hbox{\strut fail}} \\
         \hline
         \textit{Optimized} & 233.8 & 88.4\% & 10.5\% & 1.6 & 221.1 & 86.7\% & 10.5\% &2.95 & 254.5 & 92.2\% & 6.6\% &2.2\\ 
         \textit{Narrow} & 183.0 & 72.4\% & 30.0\% & 1.0 & 149.1 & 65.1\% & 32.4\% & 1.2 & 220.4 & 83.3\% & 15.9\% & 1.3\\
         \textit{Wide} & 210.6 & 81.4\% & 17.6\% & 3.4 & 153.9 & 75.1\% & 14.0\% & 4.05 & 119.9 & 67.4\% & 17.6\% & 4.0\\
         \textit{No-interv.} & 236.3 & 90.1\% & 7.9\% & 1228.8 & 210.3 & 85.5\% & 11.8\% & 1651.9 & 248.7 & 92.4\% & 6.1\% & 1368.1
     \end{tabular}
     \vspace{0.1cm}
    \caption{\small \textit{Lunar Lander} final performance summary. \textbf{Noiseless, 2-layer student (Left):} The \textit{Wide} and \textit{Narrow} interventions result in low performance due to the difficulty of exploration and early plateau, respectively.
    Students that learn under the \textit{Optimized} curriculum policy achieve a comparable performance to those training under \textit{No-intervention}, but suffer three orders of magnitude fewer training failures. \textbf{Noisy student (Center), One layered-student (Right):} The results are similar when we use the curriculum optimized for students with noiseless observations and a 2-layer MLP policy for students with noisy sensors or a 1-layer architecture.
}
    \label{tab:lander}
    \vspace{-0.8cm}
\end{table}
\endgroup
\backupmedium
\section{Concluding remarks}
\backupsmall
In this work, we introduce \cisr, a novel framework for safe RL that avoids many of the impractical assumptions common in the safe RL literature. In particular, we introduce curricula inspired by human learning for safe training and deployment of RL agents and a principled way to optimize them. Finally, we show how training under such optimized curricula results in performance comparable or superior to training without them, while greatly improving safety.

While it yields promising results, the teaching policy class considered here is quite simplistic. The use of more complex teaching policy classes, which is likely to require a similarity measure between interventions and techniques to reduce the teacher's sample complexity, is a relevant topic for future research.
Finally, it would be relevant to study how to combine the safety-centered interventions we propose with other task-generating techniques that are used in the field of curriculum learning for RL.

\clearpage

\begin{ack}
    We would like to thank Cathy Wu (MIT), Sanmit Narvekar (University of Texas at Austin), Luca Corinzia (ETH Zurich) and Patrick MacAlpine (Microsoft Research) for their comments and suggestions regarding this work. Alekh Agarwal and Andrey Kolobov also thank Geoff Gordon, Romain Laroche and Siddhartha Sen for formative discussions. This work was supported by the Max Planck ETH Center for Learning Systems. This project has received funding from the European Research Council (ERC) under the European Unions Horizon 2020 research and innovation program grant agreement No 815943. 
\end{ack}
\section*{Broader Impact Statement}
\looseness -1 Our paper introduces conceptual formulations and algorithms for the safe training of RL agents. RL has the potential to bring significant benefits for society, beyond the existing use cases. However, in many domains it is of paramount importance to develop RL approaches that avoid harmful side-effects during learning and deployment. We believe that our work contributes to this quest, potentially bringing RL closer to high-stakes real-world applications.
Of course, any technology -- especially one as general as RL -- has the potential of misuse, and we refrain from speculating further.


\bibliographystyle{apalike}
\bibliography{neurips20}

\clearpage

\appendix


\begin{center}
\textbf{\Large{APPENDIX}}
\end{center}

\section{Hyperparameters} \label{sec:app:hyperpar}
In this section, we report the hyperparameters that we use for the students, which are CMDP solvers based on an online version of \cite{le2019batch}, and for the teachers, which are based on the GP-UCB algorithm for multi-armed bandits \cite{srinivas2009gaussian}.

\subsection{Students} The students comprise two components: an unconstrained RL solver and a no-regret online optimizer. The first component is used to solve the unconstrained RL problem that results from optimizing the Lagrangian of a given CMDP for a fixed value of the Lagrange multipliers. For this, we use the Stable Baselines \cite{stable-baselines} implementation of the Proximal Policy Optimization (PPO) algorithm \cite{schulman2017proximal}. The second component is used to adapt the Lagrangian multipliers online. As suggested in \cite{le2019batch}, we use the Exponentiated Gradient algorithm \cite{kivinen1997exponentiated} for this. In the following, we use the hyperaparameters naming convention from Stable Baselines \cite{stable-baselines} for PPO and from \cite{le2019batch} for Exponentiated Gradient.

\mypar{Frozen Lake} In \cref{tab:flake_hyperpar_student}, we show the hyperparameters used for the students in the \textit{Frozen Lake} experiments, except for those that determine their policy class. In these experiments, the student's policies are parametrized as convolutional neural networks with 2 convolutional layers followed by a fully connected layer. The first convolutional layer has 32 filters of size 3 and stride 1. The second one has 64 filters of size 3 and stride 1. The fully connected layers contains 32 neurons. We use ReLU as activation function.

\mypar{Lunar Lander} In \cref{tab:lander_hyperpar_student}, we show the hyperparameters used for the students in the \textit{Lunar Lander} experiments, except for those that determine their policy class. In these experiments, the student's policies are parametrized as MLP networks with 2 hidden layers with 20 neurons each. We use ReLU as activation function.

\subsection{Teachers}
Our teachers are based on the the GPyOpt \cite{gpyopt2016} implementation of GP-UCB.

The teacher's hyperparameters are those of the Gaussian process (GP) model used by GP-UCB. In all the experiments, we use a GP with radial basis function (RBF) kernel with automatic relevance determination (ARD) and a Gaussian likelihood. Therefore, the teacher has the following hyperparameters: the signal variance, $\sigma^2_f$, an array of $3K+1$ lengthscales $l \in \mathbb{R}^{3K+1}$, where $3K+1$ is the number of parameters that determines the teacher's policy for a fixed number of intervention switches, $K$, see \cref{sec:impl}, and the noise variance $\sigma^2_n$. Rather than fixing the hyperparameters a priori, we define  hyperpriors over them and use their maximum a posteriori (MAP) estimate, which we update after every newly acquired data point.

The data is normalized before being fed to the GP.

\mypar{Frozen Lake}
In the \textit{Frozen Lake} experiments, we allow for up to two switches between interventions; that is, $K=2$. Therefore, $l\in \mathbb{R}^7$. In \cref{tab:flake_hyperpar_teacher}, we show the mean and the variance of the Gamma hyperprior of each hyperparameter.

\mypar{Lunar Lander}
In the \textit{Lunar Lander} experiments, we allow for up to one switch between interventions; that is, $K=1$. Therefore, $l\in \mathbb{R}^4$. In \cref{tab:lander_hyperpar_teacher}, we show the mean and the variance of the Gamma hyperprior of each hyperparameter.


\begin{table}[h]
    \begin{subtable}{0.45\textwidth}
    \centering
    \begin{tabular}[t]{l|c}
        Name & Value \\
        \hline
         $n\_steps$ & 128 \\
         $ent\_coef$ & 0.05\\ 
         $learning\_rate$ & 0.001 \\
         $noptepochs$ & 9\\
    \end{tabular}
    \caption{PPO}
    \end{subtable}
    \quad
    \begin{subtable}{0.45\textwidth}
    \centering
    \begin{tabular}[t]{l|c}
        Name & Value \\
        \hline
         $B$ & 0.5 \\ 
         $\eta$ & 1.0
    \end{tabular}
    \caption{Exponentiated Gradient}
    \end{subtable}
    \caption{Student's hyperparameters for the \textit{Frozen lake} environment.}
    \label{tab:flake_hyperpar_student}
\end{table}

\begin{table}[h]
    \begin{subtable}{0.45\textwidth}
    \centering
    \begin{tabular}[t]{l|c}
        Name & Value \\
        \hline
         $n\_steps$ & 500 \\
         $ent\_coef$ & 0.001\\ 
         $learning\_rate$ & 0.005 \\
         $noptepochs$ & 32\\
    \end{tabular}
    \caption{PPO}
    \end{subtable}
    \quad
    \begin{subtable}{0.45\textwidth}
    \centering
    \begin{tabular}[t]{l|c}
        Name & Value \\
        \hline
         $B$ & 120 \\ 
         $\eta$ & 1.0
    \end{tabular}
    \caption{Exponentiated Gradient}
    \end{subtable}
    \caption{Student's hyperparameters for the \textit{Lunar Lander} environment.}
    \label{tab:lander_hyperpar_student}
\end{table}

\begin{table}[h]
    \centering
    \begin{tabular}{c|c c c c c c c c c}
        Hyperparameter &  $\sigma^2_f$ & $l_1$ & $l_2$ & $l_3$ & $l_4$ & $l_6$ & $l_5$ & $l_7$ & $\sigma^2_n$ \\
        \hline
        $\mu$ & 1 &  1 & 0.05 & 1 & 0.05 & 0.2 & 0.2 & 0.2 & 0.01 \\
        $\sigma^2$ & 0.2 & 1 & 0.02 & 1 & 0.02 & 0.2 & 0.2 & 0.2 & 0.1
    \end{tabular}
    \caption{Mean and variance of the Gamma hyperpriors for the teacher's hyperparameters for the \textit{Frozen Lake} environment.}
    \label{tab:flake_hyperpar_teacher}
\end{table}

\begin{table}[ht]
    \centering
    \begin{tabular}{c|c c c c c c}
        Hyperparameter &  $\sigma^2_f$ & $l_1$ & $l_2$ & $l_3$ & $l_4$ & $\sigma^2_n$ \\
        \hline
        $\mu$ & 1 &  20 & 1 & 0.2 & 0.2 & 0.01 \\
        $\sigma^2$ & 0.2 & 4 & 0.3 & 0.2 & 0.2 & 0.1
    \end{tabular}
    \caption{Mean and variance of the Gamma hyperpriors for the teacher's hyperparameters for the \textit{Lunar Lander} environment.}
    \label{tab:lander_hyperpar_teacher}
\end{table}
\section{Experiments}
\label{sec:app:environments}
In this section, we provide a detailed explanation of our experimental setup and we present the results we obtained repeating the curriculum optimization and evaluation for multiple random seeds.

\subsection{Frozen Lake}
\mypar{Environment}
In the \textit{Frozen Lake} experiments, we use the $10\times 10$ map in \cref{fig:flake_small_map}. The student receives the full map as observation but is not aware of the areas of influence of the teacher, \cref{fig:flake_small_map_with_teacher1,fig:flake_small_map_with_teacher2}. In each location, it can take one of four actions: \textit{up}, \textit{right}, \textit{left} or \textit{down}. With probability $80\%$ it moves in the desired direction and with $10\%$ probability it moves in either of the orthogonal ones. After each move, it can end up in one of three kind of tiles: \textit{goal}, which results in a successful termination of the episode, \textit{danger}, which results in a  failure and the consequent termination of the episode and \textit{safe}. The agent receives a reward of $6$ for reaching the goal and $-0.01$ otherwise (entering dangerous tiles is discouraged via the constraint rather than with low rewards).

An interaction unit between the student and the teacher consists of 10000 time steps. A curriculum lasts for 11 of such interaction units. 

\mypar{Teacher's training} We consider curriculum policies that allow for up to two intervention switches, i.e., $K=2$. To initialize the GP model, we sample 10 curriculum policies at random, train a student with each of those and feed their final performance to the GP model. To optimize the curriculum, we run GP-UCB for 20 iterations, where each iteration corresponds to training a single student with the curriculum policy proposed by GP-UCB.

\mypar{Teacher's evaluation} To evaluate the quality of a curriculum policy, we get 10 new students, we let them train with the curricula induced by such policy and we record their failures during training as well as their returns and their successes when they are deployed in the original environment (i.e. without supervision) for 10000 time steps. \cref{fig:flake_curves} reports the mean of these quantities over the 10 students for all the curriculum policies that we consider at the end of each interaction unit. Notice that evaluating after each intervention unit is done solely for analysis purposes as, in practice, one should not deploy a student in the original environment before the curriculum is completed.

For each of the teaching policies considered, we report the mean returns and success rates at the end of the curriculum over the 10 students as well as their mean number of failures during training in \cref{tab:flake}. Here, the confidence intervals are obtained by optimizing 3 curriculum policies independently with different random seeds and repeating the evaluation procedure for each one.

\begin{table}[t]
    \centering
    \begin{tabular}{l|c c c}
         & Success & Training failures & $V_{\M}(\pi_{N_s})$ \\
         \hline
         \textit{Optimized} & $\mathbf{0.960\pm0.004}$ & $0\pm0$             & $\mathbf{5.368\pm0.025}$  \\
         \textit{SR2}       & $0.827\pm0.027$ & $0\pm0$             & $4.669\pm0.168$  \\
         \textit{SR1}       & $0.850\pm0.011$ & $0\pm0$             & $4.839\pm0.065$  \\
         \textit{HR}        & $0.000\pm0.000$ & $0\pm0$             & $-0.222\pm0.013$ \\
         \textit{Bandit}  & $0.574\pm0.049$ & $0\pm0$    & $3.077\pm0.288$  \\
         \textit{No-interv.}  & $0.768\pm0.028$ & $3075.6\pm492.1$    & $4.329\pm0.160$  \\
     \end{tabular}
    \caption{Final deployment performance in \textit{Frozen Lake} with confidence intervals obtained by training and evaluating the teachers with three different random seeds. The students trained with the optimized curriculum outperform both naive curricula and training in the original environment in terms of success rate and return. All the agents supervised by a teacher are safe during training. In contrast, training directly in the original environment results in many failures. These results are consistent across random seeds, thus showing the robustness of \cisr.}
    \label{tab:flake}
\end{table}

\subsection{Lunar Lander}

\begin{table}[t]
    \begin{subtable}[h]{\textwidth}
    \centering
        \begin{tabular}{l|c c c c c}
        & Succ. & Crash & OOM & $V_{\M}(\pi_{N_s})$ & \vtop{\hbox{\strut Training}\hbox{\strut failures}}\\
        \hline
        \textit{Optimized} & $89.2\pm0.4\%$ & $9.5\pm0.6\%$ & $0.4\pm0.1\%$ & $236.1\pm0.9$ & $1.7\pm0.13$     \\
         \textit{Wide}     & $79.2\pm2.2\%$ & $18.7\pm1.5\%$ & $0.4\pm0.1\%$ & $199.5\pm8.3$ & $3.5\pm0.15$     \\
         \textit{Narrow}   & $72.7\pm3.1\%$ & $23.4\pm2.1\%$ & $0.4\pm0.2\%$ & $187.2\pm9.0$ & $0.9\pm0.02$     \\
        \textit{No-interv.}& $88.5\pm1.6\%$ & $7.3\pm0.1\%$ & $1.9\pm1.4\%$ & $225.1\pm7.4$ & $1251.0\pm33.72$ \\
        \end{tabular}
    \caption{Evaluation on noiseless, 2-layer students.}
    \label{tab:TwoLayeredNoiselessStudentCI}
    \end{subtable}
    \\
    \begin{subtable}[h]{\textwidth}
        \centering
        \begin{tabular}{l|c c c c c}
        & Success & Crashes & OOM & $V_{\M}(\pi_{N_s})$ & \vtop{\hbox{\strut Training}\hbox{\strut failures}}\\
        \hline
        \textit{Optimized} & $83.4\pm2.3\%$ & $13.2\pm2.2\%$ & $0.3\pm0.1\%$ & $211.5\pm6.8$  & $2.6\pm0.23$     \\
        \textit{Wide} & $78.8\pm2.6\%$ & $13.8\pm0.8\%$ & $0.7\pm0.4\%$ & $184.7\pm21.8$ & $4.2\pm0.12$     \\
        \textit{Narrow} & $63.2\pm1.3\%$ & $32.9\pm1.0\%$ & $0.6\pm0.2\%$ & $139.1\pm7.1$  & $1.8\pm0.46$     \\
        \textit{No-interv.} & $86.0\pm0.4\%$ & $10.8\pm0.1\%$ & $0.8\pm0.2\%$ & $214.7\pm3.1$  & $1695.8\pm31.04$ \\
        \end{tabular}
        \caption{Evaluation on 2-layer students with noisy sensors.}
        \label{tab:TwoLayeredNoisyStudentCI}
    \end{subtable}
    \\
    \begin{subtable}[h]{\textwidth}
        \centering
        \begin{tabular}{l|ccccc}
        & Success & Crashes & OOM & $V_{\M}(\pi_{N_s})$ & \vtop{\hbox{\strut Training}\hbox{\strut failures}}\\
        \hline 
        \textit{Optimized} & $92.1\pm1.6\%$ & $5.1\pm0.6\%$ & $0.0\pm0.0\%$ & $253.4\pm5.0$  & $1.9\pm0.13$     \\
        \textit{Wide}      & $72.2\pm2.9\%$ & $16.6\pm0.9\%$ & $2.4\pm1.1\%$ & $151.6\pm18.5$ & $3.5\pm0.25$     \\
        \textit{Narrow}    & $81.7\pm1.0\%$ & $16.3\pm0.2\%$ & $0.0\pm0.0\%$ & $221.0\pm2.1$  & $1.3\pm0.03$     \\
        \textit{No-interv.}& $94.5\pm1.0\%$ & $4.5\pm0.9\%$ & $0.1\pm0.0\%$ & $256.4\pm3.7$  & $1175.0\pm80.56$ \\
        \end{tabular}
        \caption{Evaluation on noiseless, 1-layer students.}
        \label{tab:OneLayeredStudentCI}
    \end{subtable}
    \caption{\small \textit{Lunar Lander} final deployment performance summary for three different kinds of students with confidence intervals obtained by training and evaluating the teachers with three different random seeds. \textbf{Noiseless, 2-layer student (Top):} The \textit{Narrow} intervention helps exploration but results in policy performance plateau, the \textit{Wide} one slows down student learning due to making exploration more challenging, and the \textit{Optimized} teacher provides the best of both by switching between \textit{Narrow} and \textit{Wide}. Students that learn under the \textit{Optimized} curriculum policy achieve a comparable performance to those training under \textit{No-intervention}, but suffer three orders of magnitude fewer training failures. \textbf{Noisy student (Center), One layered-student (Bottom):} The results are similar when we use the curriculum optimized for students with noiseless observations and a 2-layer MLP policy for students with noisy sensors (center) or a 1-layer architecture (bottom), thus showing teaching policies can be transferred across classes of students. These results are consistent across random seeds, thus showing the robustness of \cisr.}
    \label{tab:landerCI}
\end{table}
The observation space of the \textit{Lunar Lander} environment is 8-dimensional and it includes: $x$ and $y$ position, tilt angle, linear and angular velocities and two Booleans that indicate whether each leg is in contact with the ground. At each time step, the lander can take one of four actions: fire the main, the left or the right engine or do nothing. The agent receives a reward of $100$ for a successful landing, of -0.3 for firing the main engine and of -0.03 for firing the side engines. Additionally, there is a potential based reward shaping that encourages the contact of the legs with the ground and moving toward the origin, i.e., the center of the landing pad. At the beginning of each episode, a random force is applied to the agent and the surface of the Moon is generated at random, with the only constant being the flat surface of the landing pad in the center of the map. Since the agent does not observe its distance from the ground and since the surface of the Moon is generated at every episode, the only way to guarantee safety is to land on the landing pad. In the original environment, each episode can terminate with either a successful landing or a failure (either a crash or exiting the game window from the sides, which we call an out of map, OOM, outcome). However, since the teacher's interventions make it hard for inexperienced students to come across a natural ending of the episode, we introduce a timeout, which we set to 500 during training and to 2000 during deployment (a well trained agent usually requires between 150 and 250 steps to land). Every time an episode ends because of a timeout, the student receives a reward of -100.

The trigger function of the interventions depends on whether the student's is above the landing pad, i.e., $x\in \mathcal{X}_{\textrm{land}}$, or not. In particular, let us denote with $x$ and $y$ the position of the agent, with  $\dot{x}$ and $\dot{y}$ its linear velocities, with $\alpha$ its tilt angle and with $\dot{\alpha}$ its angular velocity. The landing pad stretches between $-0.2$ and $0.2$, while the whole map extends from $-1$ to $1$. For a fixed steepness of the funnel $a$, the trigger function of the interventions are of the form:
\begin{equation}
  trigger(x,y,\dot{x},\dot{y},\alpha,\dot{\alpha}) =
    \begin{cases}
      \mathbb{I}(\dot{y}\geq 0.3 + 10 y) \lor \mathbb{I}(\alpha\geq 0.5 + 10 y) & \text{if $x\in [-0.2, 0.2]$}\\
      y \leq a(-0.2 - x) & \text{if $x< -0.2$}\\
      y \leq a(x - 0.2) & \text{if $x> -0.2$}
    \end{cases}       
\end{equation}

The reset distribution that determines the student's state after the teacher intervenes also depends on whether the teacher rescues the student above the landing pad or not. We denote with $(x,y,\dot{x},\dot{y},\alpha,\dot{\alpha})$ the state where the students gets rescued and with $(x',y',\dot{x}',\dot{y}',\alpha',\dot{\alpha}')$ the state where the student gets reset. First of all, the teacher always stabilizes the student and, therefore, we have $\dot{x}'=\dot{y}'=\alpha'=\dot{\alpha}'=0$. Thus, the reset distributions only differ based on the location where the teacher steers the student to make it stay clear from danger. In particular, if $x \in [-0.2, 0.2]$, we have $x'=x$ and $y'=y-0.1$. However, if $x > 0.2$, the reset location of the student is determined by a geometric construction: we reset the student at the intersection between the line of that forms and angle of $135^\circ$ with the horizontal axis passing through $x$ and $y$ and the line $a'(x - 0.2)$, for a given $a'>a$ (the orange line in \cref{fig:funnel}). A symmetric construction is used in case $x<-0.2$.

The \textit{Narrow} intervention corresponds to $a=20$ and $a'=100$, while the \textit{Wide} intervention corresponds to $a=0.5$ and $a'=1$.

An interaction unit between the student and the teacher consists of 100000 time steps. A curriculum lasts for 15 of such interaction units. 

\mypar{Teacher's training} We consider curriculum policies that allow for up to one intervention switch, i.e., $K=1$. Since the student's learning dynamics are quite noisy in this environment, we evaluate each curriculum policy for a class of 10 students in parallel and use the mean final performance of the students as a signal for GP-UCB. To initialize the GP model, we use 4 curriculum policies, one for each possible combination of interventions allowed by the policy class considered. To optimize the curriculum, we run GP-UCB for 10 iterations, where each iteration corresponds to training a class of 10 students in parallel with the curriculum policy proposed by GP-UCB.

\mypar{Teacher's evaluation} The evaluation of teaching policies is analogous to the \textit{Frozen Lake} case: we let each curriculum policy train 10 newly sampled students and we deploy them in the original environment for 200000 time steps to measure their performance. Since a much longer deployment time compared to \textit{Frozen Lake} is required to obtain accurate estimates of the student's performance, we only record it at the end of the curriculum rather than after each interaction unit. 

In these experiments, we also investigate the transferability of teaching policies, which is of great importance for many practically relevant scenarios. To this end, we apply the teaching policy optimized for students with perfect state information to students with noisy sensors. In particular, we consider students that observe $\tilde{x}=x+w_x$ and $\tilde{y}=y+w_y$, where $w_x\sim \mathcal{N}(0, 10^{-4})$ and $w_y\sim \mathcal{N}(0, 10^{-4})$. This level of noise is quite challenging as one standard deviation covers $2.5\%$ of the width of the landing pad. In these experiments, the teacher uses the noiseless state information to rescue the student. This captures a scenario that is common in real-world applications where we have hardware that helps preserving safety during training, such as motion capture systems, that we cannot use during deployment. Since training in these conditions is harder and more prone to constraint violation, we let the training run for 20 interaction units rather than 15 and we allow for higher penalty for constrain violation through the Lagrange multipliers by considering a higher upper bound on them (we set $B=160$ rather than $B=120$).

In a separate experiment, we apply the teaching policies optimized for the student's architecture presented in \cref{sec:app:hyperpar} to students that only have one hidden layer with 20 neurons rather than two.

For each of the curriculum policies considered and for each of the experiments described above, we report the mean returns, success rates and failure rates at the end of the curriculum over the 10 students as well as their mean number of failures during training  in \cref{tab:landerCI}. Notice that the fact that the rates do not sum to $100\%$ is due to timeouts.The confidence intervals in \cref{tab:TwoLayeredNoiselessStudentCI,tab:TwoLayeredNoisyStudentCI,tab:OneLayeredStudentCI} are obtained by optimizing 3 curriculum policies independently with different random seeds and repeating the evaluation procedure for each one.

\section{Proof} \label{sec:app:proof}
In this section, we provide proofs for \cref{prop:feasibility,prop:learning_safety}.\

\SafeDeployment*
\begin{proof}
    The main idea of the proof is to show that the constraints in Equation \cref{eq:CMDP_intervention}, which are based on expectations with respect to $\rho_i^\pi$, are stricter than the constraint in Equation \cref{eq:CMDP}, which are based on expectations with respect to $\rho^\pi$. To this end, we need to distinguish between trajectories, or segments thereof, that have the same probability under $\rho^\pi$ and $\rho^\pi_i$ for any $\pi$ from those that do not.
    Let us denote with $\xi=(s_0,s_1,\ldots,s_T)$, a generic trajectory in $\M$ and with $\Xi$ the set of all possible trajectories in $\M$ (this is the set that the distributions $\rho$ and $\rho_i$ are defined over). Moreover, for a given set of trigger states $\D_i$, we indicate the set of trajectories where at least one state belongs to $\D_i$ with $\Xi_{\D_i}=\{\xi \in \Xi~|\xi \cap \D_i \neq \emptyset\}$ and with $\Xi_{\D_i}^C=\Xi \setminus \Xi_{\D_i}$ its complement. With this notation, the constraint $\mathbb{E}_{\rho^\pi_i} \sum_{t=0}^T \mathbb{I}(s_t \in \D_i) \leq \tau_i$ is equivalent to $\sum_{\xi \in \Xi} \rho^\pi_i (\xi)|\xi \cap \D_i| \leq \tau_i$. Therefore, for a $\pi \in \Pi_{\M_i}$, we know that:
    \begin{equation}
        \sum_{\xi \in \Xi_{\D_i}} \rho^\pi_i (\xi)|\xi \cap \D_i| + \sum_{\xi \in \Xi_{\D_i}^C} \rho^\pi_i (\xi)|\xi \cap \D_i| \leq \tau_i. \label{eq:eventual_safety:teacher_condition}
    \end{equation}
    Since, by definition, we know that $|\xi \cap \D_i|=0$ for all $\xi \in \Xi_{\D_i}^C$, \cref{eq:eventual_safety:teacher_condition} simplifies to 
    \begin{equation}
        \sum_{\xi \in \Xi_{\D_i}} \rho^\pi_i (\xi)|\xi \cap \D_i|  \leq \tau_i.
    \end{equation}
    Every trajectory $\xi \in \Xi_{\D_i}$ can be divided in two segments: $\xi_1=(s_0,s_1,\ldots,s_m)$, which contains all the states up to the first one in the sequence that belongs to $\D_i$, i.e., $s_0,s_1,\ldots,s_{m-1} \not \in \D_i$ and $s_m \in D_i$, and $\xi_2=(s_{m+1},\ldots,s_T)$, which contains the remaining part of the trajectory. Thus, we can say:
    \allowdisplaybreaks
    \begin{align}
        \tau_i &\geq \sum_{\xi \in \Xi_{\D_i}} \rho^\pi_i (\xi)|\xi \cap \D_i|,\\
        & = \sum_{(\xi_1,\xi_2) \in \Xi_{\D_i}} \rho^\pi_i (\xi_1,\xi_2)|(\xi_1,\xi_2) \cap \D_i|,\\
        & \geq \sum_{(\xi_1,\xi_2) \in \Xi_{\D_i}} \rho^\pi_i (\xi_1,\xi_2)|\xi_1 \cap \D_i|,\\
        & = \sum_{\xi_1 \in \Xi_{\D_i}} \rho^\pi_i (\xi_1)|\xi_1 \cap \D_i|\sum_{\xi_2 \in \Xi_{\D_i}} \rho^\pi_i(\xi_2|\xi_1),\\
        & = \sum_{\xi_1 \in \Xi_{\D_i}} \rho^\pi (\xi_1)|\xi_1 \cap \D_i|\sum_{\xi_2 \in \Xi_{\D_i}} \rho^\pi(\xi_2|\xi_1),
        \label{eq:eventual_safete:change_distribution}
        \\
        & = \sum_{(\xi_1,\xi_2) \in \Xi_{\D_i}} \rho^\pi (\xi_1,\xi_2)|\xi_1 \cap \D_i|,
        \label{eq:eventual_safete:change_di_d}
        \\
        & \geq \sum_{(\xi_1,\xi_2) \in \Xi_{\D_i}} \rho^\pi (\xi_1,\xi_2)|(\xi_1,\xi_2) \cap \D| = \sum_{\xi \in \Xi_{\D_i}} \rho^\pi (\xi)|\xi \cap \D|. \label{eq:eventual_safety:teacher_condition_transformed}
    \end{align}
    In the previous chain of inequalities, \cref{eq:eventual_safete:change_distribution} holds because $\rho$ and $\rho_i$ are the same for the portion of the trajectory before the teacher intervenes for the first time, i.e., $\xi_1$, and because $\sum_{\xi_2 \in \Xi_{\D_i}} \rho^\pi_i(\xi_2|\xi_1)=\sum_{\xi_2 \in \Xi_{\D_i}} \rho^\pi(\xi_2|\xi_1)=1$. Furthermore, \cref{eq:eventual_safete:change_di_d} holds because $|\xi_1\cap \D_i|=1$ by definition of $\xi_1$ and because $|(\xi_1,\xi_2) \cap \D|\leq 1$ since the states in $\D$ are terminal.
    
    Moreover, for a $\pi \in \Pi_{\M_i}$, we know that:
    \begin{equation}
        \sum_{\xi \in \Xi_{\D_i}} \rho^\pi_i (\xi)|\xi \cap \D| + \sum_{\xi \in \Xi_{\D_i}^C} \rho^\pi_i (\xi)|\xi \cap \D| \leq \kappa_i. \label{eq:eventual_safety:danger_condition}
    \end{equation}
    Since the teacher does not modify the original dynamics unless the student enters a trigger state, we know that $\rho^\pi_i=\rho$ for every $\xi \in \Xi_{\D_i}^C$. Therefore, \cref{eq:eventual_safety:danger_condition} becomes:
    \begin{equation}
        \sum_{\xi \in \Xi_{\D_i}} \rho^\pi_i (\xi)|\xi \cap \D| + \sum_{\xi \in \Xi_{\D_i}^C} \rho^\pi (\xi)|\xi \cap \D| \leq \kappa_i. \label{eq:eventual_safety:danger_condition_transformed}
    \end{equation}
    By summing \cref{eq:eventual_safety:teacher_condition_transformed,eq:eventual_safety:danger_condition_transformed}, we obtain:
    \begin{equation}
        \sum_{\xi \in \Xi_{\D_i}} \rho^\pi (\xi)|\xi \cap \D| + \sum_{\xi \in \Xi_{\D_i}^C} \rho^\pi (\xi)|\xi \cap \D| \leq \tau_i + \kappa_i - \sum_{\xi \in \Xi_{\D_i}} \rho^\pi_i (\xi)|\xi \cap \D| \leq \tau_i + \kappa_i,
    \end{equation}
    which means $\mathbb{E}_{\rho^\pi} \sum_{t=0}^T \mathbb{I}(s_t \in \D) \leq \tau_i + \kappa_i\leq \kappa$, which implies that $\pi \in \Pi_\M$.
\end{proof}

\SafeTraining*
\begin{proof}
    During learning, the student may use any policy $\pi$ from the set of all possible policies for the origianl environment $\M$, $\overline{\Pi}_\M$ (which includes unfeasible and, therefore, unsafe policies). However, during training, for any intervention $i$ the student transitions according to the dynamics $\Pp_i$ rather than $\Pp$. Therefore, we aim to show that $\mathbb{E}_{\rho^\pi_i}\sum_{t=0}^T \mathbb{I}(s_t \in D) \leq \kappa$ for all $\pi \in \overline{\Pi}_\M$. If $s_t\in \Ss\setminus \D$, we can either have (\textit{i}) $s_t\in \Ss\setminus \D_i$ or (\textit{ii}) $s_t \in \D_i \setminus \D$. Let us consider these two cases separately. In (\textit{i}), we know that $\Pp(s_{t+1}|s_t,a_t)=0$ for any action $a_t$ and any $s_{t+1} \in \D$ by assumption. Moreover, since $\Pp(s_{t+1}|s_t,a_t)=\Pp_i(s_{t+1}|s_t,a_t)$ for all $s_t \not \in \D_i$, we have $\Pp_i(s_{t+1}|s_t,a_t)=0$ for all $s_t \in \Ss \setminus \D_i$ and $s_{t+1} \in \D$. In case (\textit{ii}), we know that $\Pp_i(s_{t+1}|s_t,a_t)=0$ for all $s_{t+1} \in \D_i \supseteq \D$ and $s_t \in \D_i$ by definition of reset distribution. Therefore, we have shown that for all $s_t \in \Ss \setminus \D$, $a_t \in \A$ and $s_{t+1} \in \D$, we have $\Pp_i(s_{t+1}|s_t,a_t)=0$. As a consequence, the only way the student can reach an unsafe state under the dynamics $\Pp_i$ is if $s_0 \in \D$, which corresponds to starting an episode in an unsafe terminal state, which only depends on the initial state distribution and not on the policy. If the initial state distribution is such that it is not possible for the student to start an episode in an unsafe state, then we have $\mathbb{E}_{\rho^\pi_i}\sum_{t=0}^T \mathbb{I}(s_t \in D) = 0 \leq \kappa$ for every $\pi \in \overline{\Pi}_\M$. Otherwise, we have assumed $\kappa$ to be such that the problem in \cref{eq:CMDP} is feasible. Therefore, it must be such that all the trajectories starting with $s_0 \in \D$ can be tolerated. 
\end{proof}

\end{document}